\documentclass{article}

\usepackage{microtype}
\usepackage{graphicx}
\usepackage{subfigure}
\usepackage{booktabs} 

\usepackage{hyperref}
\usepackage{multirow}

\usepackage[accepted]{icml2024}

\usepackage{amsmath}
\usepackage{amssymb}
\usepackage{mathtools}
\usepackage{amsthm}

\usepackage[capitalize,noabbrev]{cleveref}

\theoremstyle{plain}
\newtheorem{theorem}{Theorem}[section]

\newtheorem{lemma}[theorem]{Lemma}

\theoremstyle{definition}
\newtheorem{definition}[theorem]{Definition}
\newtheorem{assumption}[theorem]{Assumption}
\theoremstyle{remark}

\usepackage{amsmath,amsfonts,bm}
\usepackage{enumitem}
\usepackage{pifont}

\def\eqref#1{equation~\ref{#1}}

\def\1{\bm{1}}

\def\vvxi{{\bm{\xi}}}

\def\ve{{\bm{e}}}
\def\vf{{\bm{f}}}
\def\vg{{\bm{g}}}

\def\vv{{\bm{v}}}
\def\vw{{\bm{w}}}
\def\vx{{\bm{x}}}
\def\vy{{\bm{y}}}
\def\vz{{\bm{z}}}
\def\vxi{{\bm{\xi}}}
\def\vD{{\bm{\gD}}}

\DeclareMathAlphabet{\mathsfit}{\encodingdefault}{\sfdefault}{m}{sl}
\SetMathAlphabet{\mathsfit}{bold}{\encodingdefault}{\sfdefault}{bx}{n}

\def\gA{{\mathcal{A}}}

\def\gD{{\mathcal{D}}}

\def\gL{{\mathcal{L}}}

\def\gO{{\mathcal{O}}}

\newcommand{\E}{\mathbb{E}}

\newcommand{\R}{\mathbb{R}}

\DeclareMathOperator*{\argmax}{arg\,max}

\DeclareMathOperator{\Tr}{Tr}

\usepackage[textsize=tiny]{todonotes}

\icmltitlerunning{ Data Selection for Multimodal Contrastive Learning}

\begin{document}

\twocolumn[
\icmltitle{Variance Alignment Score: A Simple But Tough-to-Beat \\Data Selection Method for Multimodal Contrastive Learning}




\icmlsetsymbol{equal}{*}

\begin{icmlauthorlist}
\icmlauthor{Yiping Wang}{equal,uw}
\icmlauthor{Yifang Chen}{equal,uw}
\icmlauthor{Wendan Yan}{uw}
\icmlauthor{Kevin Jamieson}{uw}
\icmlauthor{Simon Shaolei Du}{uw}
\end{icmlauthorlist}

\icmlaffiliation{uw}{Paul G. Allen School of Computer Science \& Engineering, University of Washington, Seattle}

\icmlcorrespondingauthor{Yiping Wang}{ypwang61@cs.washington.edu}

\icmlkeywords{Machine Learning, ICML}

\vskip 0.3in
]



\printAffiliationsAndNotice{\icmlEqualContribution} 

\begin{abstract}
In recent years, data selection has emerged as a core issue for large-scale visual-language model pretraining, especially on noisy web-curated datasets. One widely adopted strategy assigns quality scores such as CLIP similarity for each sample and retains the data pairs with the highest scores. However, these approaches are agnostic of data distribution and always fail to select the most informative samples. 
To solve this problem, we propose a simple yet theoretically principled metric named \textbf{V}ariance \textbf{A}lignment \textbf{S}core (\textsf{VAS}), which has the form $\langle \Sigma_{\text{test}}, \Sigma_i\rangle$.
Here, $\Sigma_{\text{test}}$ represents the target (cross-)covariance matrix we aim to align, potentially based on prior knowledge, while $\Sigma_i$ denotes the tensor product of single or multi-modal representations for the $i$-th sample.
We further design a new data selection method that maximizes the total \textsf{VAS}.
We provide theoretical analysis in a simplified setting to demonstrate the theoretical advantage of \textsf{VAS} over random or other existing data selection.
Experimentally, applying VAS and CLIP scores together can outperform baselines by a margin of $1.3\%$ average on 38 evaluation sets for noisy dataset DataComp and $2.5\%$ on VTAB for high-quality dataset CC12M. Additionally, our ablation study also shows visual features are better than text for calculating VAS, and the related classical experimental design methods may fail under this context.
\end{abstract}

\section{Introduction}

Curating web-scale visual-language datasets for pretraining models has become a widely adopted approach to improve visual model performance. Many studies have demonstrated that, despite the use of fixed models and training methods, the final performance of a model is significantly impacted by the choice of data \cite{radford2021learning, schuhmann2022laion, cherti2023reproducible, bai2023qwen}. Consequently, various data selection strategies have been proposed to tackle the challenges posed by large datasets with varied data quality and limited computational resources. These strategies, also known as filtering, subset selection, or pruning in the literature, each focus on slightly different objectives. Some strategies prioritize screening out unhelpful samples to maximize the final model performance without budget constraints, while others aim to maximize performance compared to other randomly sampled subsets within the same budget limit. Practically, these strategies are not mutually exclusive, and in this paper, we will evaluate our proposed strategies in both contexts.
%

One line of effective approaches treats each sample independently by assigning a quality score to each individual and only training on the samples that received the highest scores.
For example, the baseline CLIP score filtering method uses a pretrained CLIP model to calculate the cosine similarity score between image and text embeddings, and then discards those pairs that are less aligned. Some follow-up works include \citet{nguyen2023improving,mahmoud2023sieve,maini2023t,fang2023data} which propose more advanced image-text alignment scoring methods by using other pretrained classification models or more carefully crafted rule-based approaches. Please refer to related works Sec.~\ref{sec: related work} for a more detailed description.

All those works, however, failed to consider the selection strategy from data distribution perspective. This disadvantage becomes more significant with limited computational resources, as high-quality data are not always the most informative or representative. In some cases, certain feature can be easily learnt under much less samples than other features, so adding more high quality data with such feature can hardly further improve the ultimate model. In other cases, certain features rarely occur in our targeted test distribution, and therefore, adding the corresponding samples will have less effect on test performance. Thus, our work proposes a \textit{data-distribution-aware} strategy that can be combined with existing strategies focused on individual data quality. 

One of the most relevant works to ours is \citet{maharana2023d2}, which also considers the example difficulty and  diversity by using forward and
reverse message passing over a constructed dataset graph.  Their approach, however, is designed for general learning frameworks for both supervised and unsupervised (i.e. contrastive) learning without considering the special impact of data distribution in \textit{multi-modal contrastive learning}. 
Another relevant strategy is using the external image-set as data prior proposed by \citet{gadre2023datacomp}, which clusters the image embedding of the candidate pool and select the samples whose cluster center is the nearest center for at least one ImageNet-1K training example. But this method is also comes from more heuristic inspiration without considering the particular learning method, and is therefore more hyper-parameter sensitive.
In contrast to both approaches, our strategy is inspired by recent theoretical progress in those contrastive representation learning works \cite{nakada2023understanding,ji2023power}. As a results, our work achieves much better average performance over the 38 downstream tasks compared to these two and is almost hyper-parameter-free.

\begin{figure*}[t]
    \centering
    \small
    \includegraphics[width=0.87\textwidth]{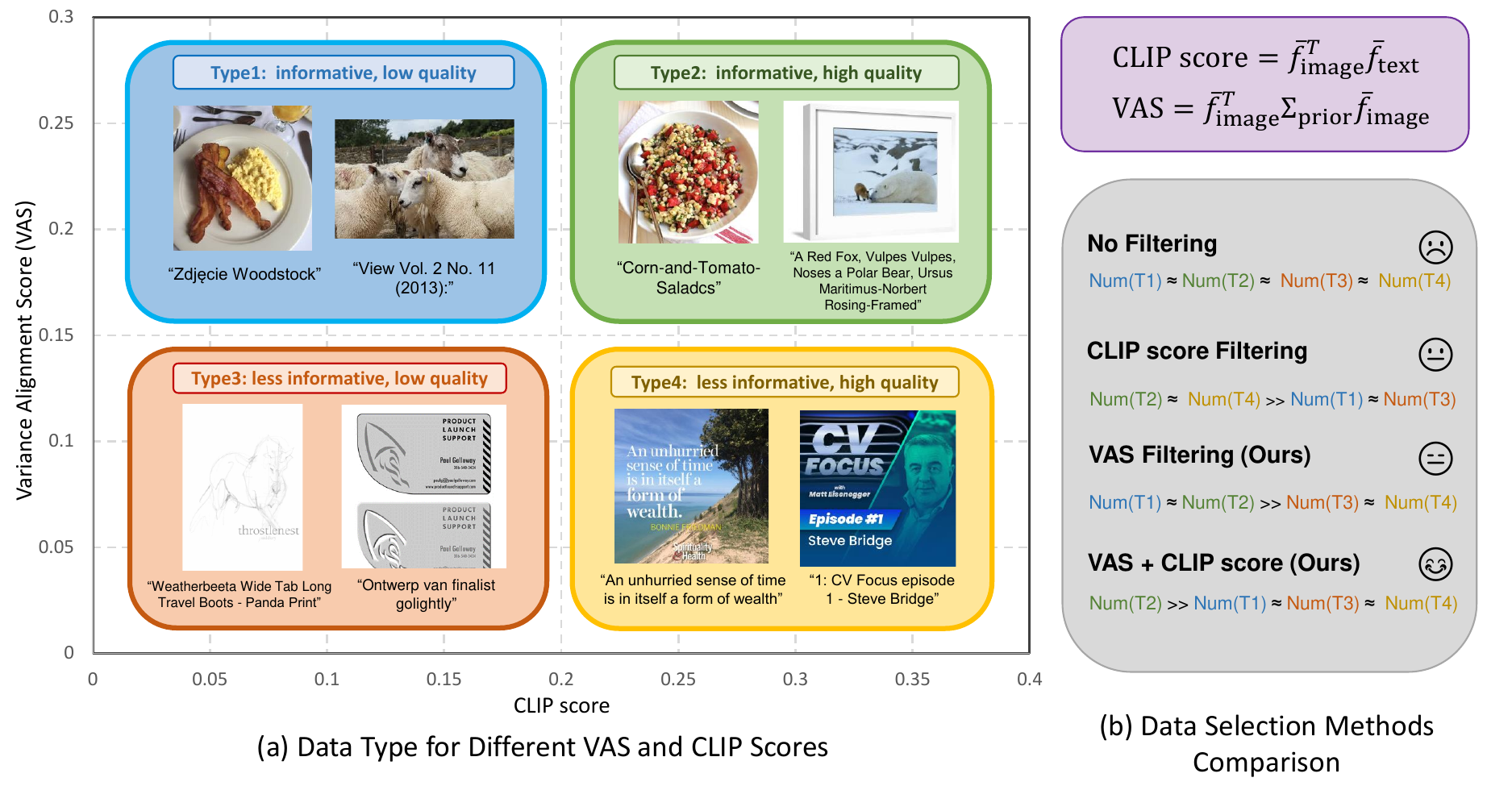}
    \vspace{-1.2em}
    \caption{Illustration of our VAS on DataComp. $\bar f$ means the embeddings calculated by the pre-trained model, and $\Sigma_{\text{prior}}$ matrix is the test prior co-variance matrix, can be chosen as the co-variance matrix of the image embeddings from ImageNet-1k or DataComp itself. (a) Visualization of data with different Variance Alignment Score (VAS) and CLIP score in DataComp. CLIP score does not efficiently evaluate the informativeness of image-text pairs. Data from OCR tasks or unconventional visual tasks (Type 4) can have high quality but little useful visual information. However, VAS can select the data with more meaningful image features (Type 1 and Type 2).
    (b) Illustration of a rough comparison of sampling data for different filtering methods. Using VAS $\cap$ CLIP score filtering can balance the informativeness and quality and increase the proportion of Type 2 data, which are the most helpful data for training image-text models. Please refer to Appendix.~\ref{sec: add_vis} for more visualization results.}
    \label{fig:vas_cs}
    \vspace{-1.0em}
\end{figure*}

Specifically, we propose \textbf{V}ariance \textbf{A}lignment \textbf{S}core (\textsf{VAS}) which aims to find the most informative subset of the training data whose (cross-)covariance between the image-text pairs most aligns with the (cross-)covariance of a reference distribution.
For example, the reference distribution could be the entirety of the pool from which a subset is being chosen, or some other more task-dependent dataset.  
We also combine VAS with a method based on an individual quality metric as it provides a boost in performance empirically. (Sec.~\ref{sec: main})
To validate the effectiveness of our strategy, we primarily leverage the DataComp benchmark \cite{gadre2023datacomp}. Compared with the existing best strategies, it improves 1.3\% averagely on 38 evaluation sets for DataComp without budget constraints and improves averagely 0.4\% for CC12M with budget constraints (see Sec.~\ref{sec: experiment}).

In addition to empirical results, we provide a theoretical justification for our strategy under a linear model assumption. Firstly, assuming we are given a perfect embedding model and proper data prior, we show that choosing the subset to minimize the multi-modal contrastive loss, including the CLIP loss, can be reduced to the problem of maximizing the total VAS score of the subset. Secondly, we prove a generalization bound when the embedding model and data prior are non-perfect (see Sec.~\ref{sec:theory}).

\vspace{-0.5em}
\section{Related Works}
\label{sec: related work}

\paragraph{Distribution-agnostic filters.} 
Raw data crawled from the internet can be highly noisy -- some samples have text failing to describe their image part, and other samples have misleading elements in their image that we don't want the model to focus on. Many strategies have been proposed to algorithmically detect those corruption.
\citet{nguyen2023improving, mahmoud2023sieve} both leverage a pre-trained visual-language model BLIP to caption the image and then compare their similarity with the original caption to detect the mismatch\footnote{Here \citet{nguyen2023improving} replaces the corrupted text with their newly generated caption instead of pruning the whole sample, but the high-level idea is similar.}. Rather than consider a sample level mapping, \citet{maini2023t} further looks into the detailed elements in the image. Their strategies remove images whose text inside is the only feature correlated with the caption by using off-the-shell OCR models. Notably, despite its effectiveness, it has a much higher computational cost than our strategy due to calling text detection models. Finally, instead of using existing models, \citet{fang2023data} proposes to train a data filtering network directly on existing high-quality data. 
\paragraph{Distribution-aware data selection.}
There exists many classical coreset selection which select out the most representative and informative data by taking data distribution into control, including \citet{wei2015submodularity,raskutti2016statistical,coleman2019selection}. All those works, however, make the assumption that we are under a fixed budget and aims to increase the subset performance reach to the pool performance. 
There also exist some recent distribution-aware selection works for CLIP pretraining. \citet{wang2023too} first trained a codebook in replace of existing pretrained models like BLIP. Then they cluster the samples based on their image corresponding codes and selects a representative subset of samples from each cluster. 
Later \citet{maharana2023d2} considers the example difficulty and data diversity by using
forward and reverse message passing over a constructed
dataset graph.
\paragraph{Dataset as prior.} 
Neither strategy mentioned above considers the distribution shift in test data. Although such downstream task is unknown in advance, it is possible to estimate via some proxy dataset as prior. This was first proposed in \citet{gadre2023datacomp} relying on sampling image text pairs that are semantically similar to diverse and curated datasets like ImageNet~\cite{deng2009imagenet}. Later \citet{xu2023cit} also use ImageNet as prior but measuring the similarity with both text and meta information. Their work focused on curation instead of subset selection but the high level idea is similar. 


\vspace{-1em}
\section{Problem Setup}

We are given a training dataset $D_\text{train} = \{\vx^{v}, \vx^{l}\}$, where $\vx^{v}, \vx^{l} \in \R^d$ is the image-text (vision-language) training pair. For convenience, we will let subscript $vl$ denote either modal so that, for example, $\vx^{vl} \in \{\vx^{v}, \vx^{l}\}$. We aim to identify a subset $S \subset D_\text{train}$, with or without the subset size constraints, that maximizes zero-shot accuracy on some downstream tasks when $S$ is used to train a CLIP model.

\paragraph{CLIP score and embedding.} Recent efforts, such as LAION \cite{schuhmann2022laion} and DataComp \cite{gadre2023datacomp} used OpenAI’s CLIP ViT-L/14 model \cite{radford2021learning} as a teacher model to obtain quality score. Here we denote this vanilla CLIP model as $\bar{f}_{vl}$.
For any pair $\vx^{vl}$, the model outputs a normalized unit-vector $\bar{f}_{vl}(\vx^{vl})$. 
Let $\bar{s}_{ij} := \langle  \bar{f}_v(\vx_i^v), \bar{f}_l(\vx_j^l) \rangle \in [-1,1]$ denote the pairwise cosine similarity between $i$-th image and $j$-th text. 
In the subsequent sections, we will refer to $\bar{s}_{ii}$ as the “CLIP score” from an existing-pretrained ``vanilla'' model publicly available. 
Intuitively, pairs with higher clip scores are more likely to be effective samples since the image is aligned with the text. However, as we will discuss later, this hypothesis is not always true. 
\vspace{-1em}
\paragraph{Dataset and model.} Here we follow the pipeline of Datacomp \cite{gadre2023datacomp} to standardize the training and evaluation process. This is a testbed for dataset experiments aiming to open-source and further improve the vanilla CLIP model and is widely adopted in previous data selection papers~ \cite{nguyen2023improving,maharana2023d2,maini2023t,fang2023data,mahmoud2023sieve,bai2023qwen}. We will give more details in Sec.~\ref{sec: experiment}.
\section{Data Filtering Strategy}
\label{sec: main}
Fig.~\ref{fig:vas_cs} gives an overview of our strategy. We give more details below.

\subsection{Variance Alignment Score (VAS)}

Most existing data selection strategies focus on individual sample quality by measuring the similarity between image and text features (e.g., CLIP score). Yet, they fall short in evaluating how each data point contributes to the model's ability to learn meaningful representations in the context of a target distribution and \emph{other} points included in the set. To measure the informativeness of an individual sample for the overall multi-modal contrastive training, we introduce a novel, data-distribution-aware metric named Variance Alignment Score (VAS) defined as
\begin{align}\label{eq:def_VAS}
    & \text{VAS}_i(m_1,m_2) = \bar f_{m_1}(\vx_i^{m_1})^\top \bar{\Sigma}_\text{test} \bar f_{m_2}(\vx_i^{m_2}) \\
    & \text{where }
    \bar{\Sigma}_\text{test} =  \E_{\vx \sim \vD_\text{test\_proxy}} \bar{f}_{m_1}(\vx^{m_1}) \bar{f}_{m_2}(\vx^{m_2})^\top \nonumber
\end{align}
Here $m_1,m_2$ are the modalities we choose and can be visual or language. And $\vD_{\text{test\_proxy}}$ is some proxy test data distribution when the actual downstream tasks are unknown, which is also called ``data prior" in some previous work as we mentioned in Section~\ref{sec: related work}. 
The specific choice of $\bar{\Sigma}_{\text{test}}$ and $m_1,m_2$ is crucial and we left details in the next section~\ref{subsec: main strategy}. Here we provide a high level intuition behind VAS.

According to the no-free-lunch theorem, it is impossible to optimize the learning of every feature equally well within a limited budget. In other words, we cannot find a universally optimal strategy for any arbitrary test distribution. Therefore, we tend to select samples that contain effective features in test distributions. Given our focus on this particular multi-modal contrastive learning framework, we leverage insights from previous theoretical generalization bounds as discussed in \citet{ji2023power,nakada2023understanding}. These works suggest that, under a good embedding model $\Bar{f}$, the variance or cross-variance term, denoted as $\frac{1}{N}\sum_i\bar{f}_{m_1}(\vx_i^{m_1}) \bar{f}_{m_2}(\vx_i^{m_2})^\top$, is a suitable criterion for indicating informativeness.
For the rigorous theoretical justification, please refer Lemma~\ref{lem: main} .

\subsection{Main Strategy}
\vspace{-0.5em}
\label{subsec: main strategy}
\begin{algorithm}
\caption{VAS-based strategy}
\label{algo: main}
\begin{algorithmic}[ht]
    \STATE \textbf{Step 1:} \textit{Conservatively} remove low CLIP score samples. (i.e., almost certain low-quality samples)
    \STATE \textbf{Step 2.1:} Choose proper $\vD_{\text{test\_proxy}}$
    \STATE \textbf{Step 2.2:} Select $S = \arg \max_{|S| = N} \sum_{i \in S}\text{VAS}_i(v,v)$
\end{algorithmic}
\end{algorithm}
\vspace{-0.8em}
We then propose a simple two-stage pruning strategy in Algo.\ref{algo: main} that considers both individual quality and overall distribution by combining the CLIP score with VAS. 

In step 1, we first remove those easily distinguishable low-quality samples, such as low-resolution images or mismatched text (Type 1 and Type 3 data in Fig.~\ref{fig:vas_cs}).
However, unlike the native CLIP score filter in \citet{gadre2023datacomp}, which greedy eliminates 70\% of the data in both DataComp and CC12M datasets, we remove much less data (e.g. around half of the data). Note the overall performance is not sensitive to small variation in removal ratio. 
Subsequently, in step 2, we select a data prior, $\vD_{\text{test\_proxy}}$, either by choosing the external ImageNet-1K or heuristically selecting a subset from the training set which only contains high-frequency features (Details in Eqn.~\ref{eq:tr_S} and \ref{eq:tr_S2}). Based on the selected $\vD_{\text{test\_proxy}}$, we compute the vision-only VAS score and select a subset of size $N$ with the highest scores. Detailed discussions and ablation study results are provided to highlight the significance of these three designs.
\vspace{-0.8em}
\paragraph{CLIP score does not indicate informativeness.} While it is intuitive that extremely low CLIP scores suggest low sample quality, the converse that high CLIP scores suggest high sample quality is not necessarily true. 
For example, if an image contains words that also appear in the caption, that may produce a high CLIP score but not necessarily the kind of relationship between image and caption that we hope to learn. 
In Fig.~\ref{fig:vas_cs}, we observe that data with similar CLIP scores can obtain very different VAS. Some high CLIP score samples can have low VAS and these samples are more related to OCR tasks and unconventional visual features which are intuitively non-helpful.
Therefore, we adopt a two-stage approach wherein the first stage we remove samples with very low CLIP scores, and then apply VAS to the remaining data to perform a more fine-grained selection.

\vspace{-0.8em}
\paragraph{Choice of $\bar{\Sigma}_\text{test}$.} Although the resulting model is evaluated over 38 different downstream tasks that are unknown to the learner in advance, we find out that using ImageNet-1k \cite{deng2009imagenet} as the test proxy achieves good performance. Similar effectiveness has also been observed in \citet{gadre2023datacomp, xu2023cit}. We conjecture that this due to the similar variance distribution between ImageNet-1k and other downstream tasks. In other words, despite different tasks, they all share similar effective features. 

Besides using the external dataset ImageNet-1k, we also propose a heuristic method called VAS-D (dynamic) which \textit{only} uses the training data itself as prior. We called it ``dynamic'' because in this method we have that $\bar{\Sigma}_\text{test}$ is dynamically updating based on selected $S$ as in Eqn.~\ref{eq:tr_S}. The high-level principle is still similar. For some hyper-parameter $N$, 
\begin{align}\label{eq:tr_S}
    & S = \arg \max_{|S| = N} \sum_{i \in S}\bar f_v(\vx_i^v)^\top \bar{\Sigma}_{\text{test}} \bar f_v(\vx_i^v) \\ 
    & \bar{\Sigma}_{\text{test}} = \frac{1}{|S|}\sum_{j \in S}\bar f_v(\vx_j^v)\bar f_v(\vx_j^v)^\top \nonumber
\end{align}
And it is obviously that (\ref{eq:tr_S}) is equivalent to
\begin{equation}\label{eq:tr_S2}
    S = \arg \max_{|S| = N} 
    \Tr\left((\sum_{j \in S}\bar f_v(\vx_j^v)\bar f_v(\vx_j^v)^\top)^2\right)
\end{equation}
Surprisingly, (\ref{eq:tr_S2}) is conflicted with the general principle in target-agnostic optimal experiment design, which emphasizes the diversity of selected data~\cite{pukelsheim2006optimal}. However, in the Experimental parts (Sec.~\ref{sec: experiment}) we will show that in the context of data selection for training CLIP, the idea of optimal experiment design does not work well.
For the theoretical justification on how the choice of $\bar{\Sigma}_\text{test}$ affects the final performance, please refer to Theorem~\ref{them: main}. 

\vspace{-0.8em}
\paragraph{Necessity of using visual-only information.} The $\vx^{vl}$ we observed comes from the high dimensional representation function mapping from low-dimensional latent vectors. The motivation behind using VAS relies on a good embedding model $\bar{f}_{vl}$ that can recover this latent vector. Hence, if there exists some perfect $\bar{f}_{vl}$, then choosing both visual and language information (i.e. $m_1,m_2 = v,l$) leads to best selection. In practice, however, we found out that visual CLIP model yield much better performance, which may suggest the text embedding is less accurate in recovery. Therefore, we choose to use vision-only information. Also, for the rigorous theoretical justification on how the performance of teacher CLIP model affect the final performance, please refer to  Theorem~\ref{them: main} in Section~\ref{sec:theory}.

\section{Theoretical Interpretation}\label{sec:theory}
In this section, we give a theoretical justification on the Step 2.1, 2.2 of our proposed strategy under the linear model assumptions when low quality image and mismatched text has already been removed.
\vspace{-0.5em}
\subsection{Theoretical Setup}

\paragraph{Training data.} For any $\vx^{v}, \vx^{l} \in \R^d$ observable image and text training pairs, we define $\vz^{v}, \vz^{l}$ to be the corresponding latent vectors which contain all semantically pertinent information about our tasks of interest. Similar to previous theoretical work \cite{nakada2023understanding},
we assume each i.i.d pair $\vz^{vl}$ follows zero-mean sub-gaussian distribution whose cross-covariance satisfies
\begin{align*}
    & \text{Cov}(\vz^v,\vz^l) 
    = \Sigma_\text{train} = \text{diag}(\sigma_1,\sigma_2, \ldots),
    & \|\vz^{vl}\| = 1
\end{align*}
and each $\vx^{vl}$ is generated based on a linear model such that
\begin{align*}
    \vx^{vl} = G^*_{vl} \vz^{vl} + \vxi^{vl}.
\end{align*}
Here $G^*_{vl} \in O_{d\times r}$ is the othonormal ground truth representation mapping from the latent vector space to the input space, and $\xi^{vl} \sim \mathcal{N}(0, I_d)$ are \textit{i.i.d.} random noise. 

Also we denote the cross covariance of any finite dataset $S'$ (e.g. the given train set $D_\text{train}$) as $\Sigma_{S'}$.

\paragraph{Test data.} For any zero-shot downstream task, we assume it shares almost same data generation process as the training set, except its the cross-covariance $\Sigma_\text{test}$ does not necessarily equal $\Sigma_\text{train}$, which necessitate the choice of $\Bar{\Sigma}_{\text{test\_proxy}}$.


\paragraph{CLIP embedding model as teacher.} Under the linear model assumption, we have a teacher model $\bar{f}_{vl} = \bar{G}_{vl}$, whose generated clip embedding can partially recover the ground truth hidden vector $\vz^{vl}$ with error. 

Formally, we say teacher has $\epsilon_{v}^n$ error if for all possible $n$ budget subsets $S \subset D_\text{train}$,
\begin{align*}
    \frac{1}{|S|} \left\| \sum_{\vx^{vl} \in S} \bar{G}_v^\top \vx^v (\vx^v)^\top \Bar{G}_v - \sum_{\vx^{vl} \in S} \vz^v (\vz^v)^\top \right\|_* \leq \epsilon_{v}^n
\end{align*}
where the same notation applies for the language modal. 
By the orthonormal assumption on the ground truth matrix $G_{vl}^*$, we see that $\bar{G}_v^\top$ is aiming to inverting the map.
In addition, we say the teacher has $\epsilon_{v*l}^n$ cross modal error
\begin{align*}
   \frac{1}{|S|}\| \sum_{\vx^{vl} \in S} \bar{G}_v^\top \vx^v (\vx^l)^\top \Bar{G}_l - \sum_{\vx^{vl} \in S} \vz^v (\vz^l)^\top \|_*  \leq \epsilon_{v*l}^n
\end{align*}
When all $\epsilon_v^n, \epsilon_l^n, \epsilon_{v*l}^n \to 0$ as $n \rightarrow \infty$, then we say the teacher is strong for both modalities. But it might also be possible that only one modal, for example, visual is strong. That is $\epsilon_v^n \to 0,  \epsilon_l^n, \epsilon_{v*l}^n \gg \epsilon_v^n$.

\paragraph{Model and training.} 
According to Lemma 4.1 in \citet{nakada2023understanding}, using the CLIP loss to optimize the linear model has approximately the same training dynamics as using the regularized linear loss. Therefore, here we assume that we are learning $G_v, G_l$ by maximizing the clip score gap between the contrastive pairs, plus a regularizer, 
\begin{align*}\label{eq:linearloss}
    & \min_{G_v, G_l} \gL_S^\rho (G_v, G_l) \\
    & := \min_{G_v, G_l} \frac{\sum_{i \in S} \sum_{j \in S}(s_{ij} - s_{ii})}{|S|(|S|-1)} 
    +  \frac{\rho}{2}\frac{|S|}{|S|-1}\|G_vG_l^\top\|_F^2
\end{align*}
where  $s_{ij} := \langle  G_v^\top \vx^v_i, G_l^\top \vx^l_j \rangle$ and $\rho > 0$ is some regularizer-related \textit{constant}.
Note that this objective maximizes self-similarity and minimizes similarity between disparate pairs. 
Note that this ``loss'' can be negative, avoiding the trivial null solution of all zeros. 
We denote this training process from any given $S$ as $G_{vl} = \gA^\rho(S)$.

\paragraph{Goal and metric.} 
Under the same principle as our training loss function, we measure the performance of any learnt $G_v, G_l$ on some downstream task with distribution $\gD_\text{test}$ as test loss $\gL_{\text{test}} (G_v, G_l):=$
\begin{align*} 
     \E_{\substack{\vx^{vl}\sim \gD_\text{test}\\\vx^{vl}_2 \sim \gD_\text{test}}} (\langle G_v^\top\vx^v, G_l^\top\vx_2^{l}\rangle- \langle G_v^\top\vx^v, G_l^\top\vx^l\rangle) 
\end{align*}
This is inspired by the following classification accuracy. Assume that the test data including $C$ class, and the class distribution is $\mathcal{C}$. For every class $c$, the training data $\vx = (\vx^v, \vx^l)$ satisfies distribution $\mathcal{P}_c$. We further assume the corresponding classification templates are $\{\vx_c\}_{c=1}^C$. Thus we define classification accuracy as
\begin{equation*}
    \text{AC}(G_v, G_l) = \E_{c,c' \sim \mathcal{C} \times  \mathcal{C}}\left[\E_{\vx_i \sim \mathcal{P}_c} \mathbf{1}[s_{ic} > s_{ic'}]\right]
\end{equation*}
Therefore our goal is to minimize its gap between the best hind-side subset, for any $\rho$, without budget constraints,
\begin{align*}
    \Delta^\rho(S) = \gL_\text{test}(\hat{G}_{vl}) - \min_{S' \in D_\text{train}} \gL_\text{test}(\gA^\rho(S')), \hat{G}_{vl} = \gA^\rho(S)
\end{align*}
\vspace{-1em}
\vspace{-1em}
\subsection{Generalization Guarantees}

We now provide theoretical guarantees and postpone our proof into Appendix~\ref{app: proof}.
\textbf{Firstly, we are going to prove the intuition behind VAS score.} 

\begin{lemma}[Intuition behind VAS]
\label{lem: main}
    With high probability at least $1-\frac{1}{|S|d}$, suppose the  hind-side best subset has at least $\underline{n}$ number of samples, then we have 
    \begin{align*}
        \Delta^\rho(S) 
        &= \underbrace{\frac{1}{\rho}\max_{S' \in D_\text{train}} \left( \Tr\left( \Sigma_{\text{test}} (\Sigma_{S'} - \Sigma_{S}) \right) \right)}_\text{VAS related term}\\
        & + \quad \underbrace{\gO \left(\sqrt{\frac{d\log(d|S|)}{\underline{n}}} + \sqrt{\frac{d\log(d|S|)}{|S|}}\right)}_\text{noise}
    \end{align*}
\end{lemma}

\begin{proof} [Proof sketch]
    \ding{182} Under the assumption that both $\vz^{vl}, \xi_{vl}$ is zero-mean, maximizing the clip score gap is equivalent to maximizing the clip score of the same sample. 
    \vspace{-3px}
    \begin{align*}
         \gL_{\text{test}} (\hat{G}_v, \hat{G}_l):= - \E_{\vx^{vl} \sim \vD_\text{test}}\langle \hat{G}_v^\top\vx^v, \hat{G}_l^\top \vx^l \rangle
    \end{align*}
    \ding{183} By minimizing the regularized training loss $\gL_S^\rho (G_v, G_l)$ using Eckart-Young-Mirsky Theorem, we get a closed form solution of $\hat{G}$ as 
    \vspace{-3px}
    \begin{align*}
        \hat G_v \hat G_l^\top \approx \frac{1}{\rho} G_v^* \Sigma_S \cdot (G_l^*)^\top + \text{noise depend on $S$}
    \end{align*}
     \ding{184} Combining the result in \ding{183} and \ding{182}, we have 
     \vspace{-3px}
    \begin{align*}
        & \gL_{\text{test}} (\hat{G}_{vl}) \approx  - \frac{1}{\rho}\Tr\left( \Sigma_{\text{test}}\Sigma_S \right) -\text{noise depend on $S$}
    \end{align*}
    The same analysis can be applied on $\min_{S' \in D_\text{train}} \gL_\text{test}(\gA(S'))$ as well. Rearranging these two equations gives us the final result.
\end{proof}

This lemma shows the the $\Delta(S)$ is depend on the VAS-related term and the noise term which comes from $\xi$. When $\underline{n}$ and $|S|$ is large enough, then the VAS-related term will become dominant. This aligns with our practice experience that the final performance is less sensitive to the small variation in the number of select data as long as that is sufficient.
Moreover, in some special cases where test distribution has identity cross-variance, then sampling by greedily choosing CLIP score might be enough.

\textbf{Now we are ready to give a proof on the choice of $\bar{\Sigma}_{\text{test}}$ and visual-only information.}
Specifically, the strategy error mainly comes from (1). The unknown test distribution shift from training. (2). The unobservable ground truth $\Sigma_S$. To tackle error (1), we assume some prior knowledge on test by using the proxy test variance $\bar{\Sigma}_{\text{test}}$. To tackle the error (2), there are two possible solutions as shown below. Based on the theoretical interpretation, we should choose different strategy based on the property of the teacher embedding model.
\vspace{-5px}
\begin{align*}
    & S_{\text{vision+language}} = \argmax_{S} \Tr\left(\bar{\Sigma}_{\text{test}} (\sum_{\vx^{vl} \in S} \bar{G}_v^\top \vx^v (\vx^l)^\top \Bar{G}_l) \right)  \\
    &  S_{\text{vision only}} = \argmax_{S} \Tr\left(\bar{\Sigma}_{\text{test}} (\sum_{\vx^{vl} \in S} \bar{G}_v^\top \vx^v (\vx^v)^\top \Bar{G}_v) \right)
\end{align*}

\begin{theorem}[Main] 
\label{them: main}
    Under the assumption of Lemma~\ref{lem: main},
    \begin{align*}
       \Delta^\rho(S)
        & \leq \text{noise} +  \frac{1}{\rho}\|\bar{\Sigma}_{\text{test}} - \Sigma_{\text{test}} \| \|\Sigma_S - \Sigma_{\text{best}}\|_* \\
        & + \frac{1}{\rho}
        \begin{cases}
            \epsilon_{v * l}^S  \quad \text{(vision+language)} \\
            \epsilon_v^S +  \sqrt{1 - \frac{1}{|S|}\sum_{i \in [S]} \langle \vz^v,\vz^l \rangle)} \quad \text{(vision only)}
        \end{cases}
    \end{align*}
\end{theorem}

Firstly, it is evident that the greater the difference between $\bar{\Sigma}_{\text{test}}$ and $\Sigma_{\text{test}}$, the less improvement we can expect. Moreover, in scenarios where $\epsilon_l$ is large (indicating lower accuracy in the language part) while $\epsilon_v$ is small (indicating higher accuracy in the vision part), it may be advisable to opt for vision-only embeddings.
However, the learner should also consider the term $\sqrt{1 - \frac{1}{|S|}\sum_{i \in S} \langle \vz^v, \vz^l \rangle}$, which represents the alignment between the ground truth visual and language latent vectors, essentially reflecting the intrinsic quality of the data. If this term is already significant, relying solely on vision information as a proxy for language information could lead to suboptimal results.

\begin{table*}[ht]
\small
\centering
\vspace{-1em}
\caption{Results on DataComp. Comparison of performance of VAS filtering with data filtering methods presented in \cite{gadre2023datacomp} and $\mathbb{D}^2$\cite{maharana2023d2}. The number after the evaluation dataset means the number of tasks contained in this class. Higher is better. Here we compare the results under any dataset size. VAS(X) means using the variance matrix from datasets X. 
}
\vspace{0.3em}
\label{tab:datacomp}
\begin{tabular}{@{}lcccccc@{}}
\toprule
\multirow{2}{*}{
\textbf{Filtering Strategy}} & \textbf{Dataset} & \textbf{ImageNet} & \textbf{ImageNet Dist. Shift} & \textbf{VTAB} & \textbf{Retrieval} & \textbf{Average} \\
 & \textbf{Size} & (1 sub-task) & (5) & (11) & (3) & (38)\\
\midrule
No filtering  & 12.8M & 2.5 & 3.3 & 14.5 & 10.5 & 13.2 \\
Text-based filtering  & 3.2M & 4.6 & 5.2 & 16.9 & 11.2 & 15.6 \\
Image-based filtering  & 3.0M & 4.3 & 4.7 & 17.8 & 11.2 & 15.8 \\
CLIP score (30\%)  & 3.8M & {5.1} & {5.5} & 19.0 & 10.8 & 17.2 \\
$\mathbb{D}^2$ Pruning (image+text)  & 3.8M & {5.1} & {5.6} & {18.2} & {11.7} & 17.0 \\
Image-based $\cap$ CLIP score (30\%) & 1.4M& 3.9& 4.5& 16.2& 8.9 & 14.4\\
\midrule
CLIP score (30\%, reproduced) & 3.8M & {4.8} & 5.3 & 17.1 & 11.5 & 15.8 \\
Image-based $\cap$ CLIP score (45\%)  & 1.9M & 4.2 & 4.6& 17.4& 10.8 & 15.5 \\
$\mathbb{D}^2$ Pruning (image+text, reproduced) & 3.8M & 4.6 & 5.2 & {18.5} & 11.1 & 16.1 \\
CLIP score (45\%) & 5.8M & 4.5 & 5.1 & 17.9 & \textbf{12.3} & 16.1 \\
\midrule
\textbf{VAS} (ImageNet-1k) 
& 3.8M & 2.4& 3.1& 14.9& 10.4& 12.7 \\
\textbf{VAS} (ImageNet-1k) $\cap$ 
CLIP score 
(45\%) 
& 3.8M & \textbf{5.2} & \textbf{5.5} & \underline{19.0} & \underline{12.2} & \textbf{17.4} \\
\textbf{VAS} (DataComp) $\cap$ CLIP score 
(45\%)
& 3.8M & \underline{5.0} & \underline{5.4} & 17.9 & {12.1} & 16.2\\
\textbf{VAS-D} (DataComp) $\cap$ CLIP score 
(45\%) 
& 3.8M & 4.7 & \underline{5.4} & \textbf{19.7} & {11.7} & \underline{17.3} \\
\bottomrule
\end{tabular}
\vspace{-1em}
\end{table*}

\begin{table*}[ht]
\small
\centering
\caption{Results for coreset selection on the high-quality dataset CC12M. Comparison of VAS filtering performance with data filtering methods presented in \cite{gadre2023datacomp} and $\mathbb{D}^2$\cite{maharana2023d2}. The number after the evaluation sets means the number of tasks contained in this class. Higher is better. And we compare the results under $2.8$M dataset size. VAS(X) means using the variance matrix from datasets X. }
\vspace{0.3em}
\label{tab:cc12m}
\begin{tabular}{@{}lcccccc@{}}
\toprule
\multirow{2}{*}{
\textbf{Filtering Strategy}} & \textbf{Dataset} & \textbf{ImageNet} & \textbf{ImageNet Dist. Shift} & \textbf{VTAB} & \textbf{Retrieval} & \textbf{Average} \\
 & \textbf{Size} & (1 sub-task) & (5) & (11) & (3) & (38) \\ \midrule
No filtering & 9.2M & 9.3 & 8.5 & 21.3 & 19.6 & 20.2 \\
CLIP score (50\%) & 4.6M & 7.8 & 8.7 & 20.3 & 18.0 & 18.8  \\
\midrule
CLIP score (30\%) & 2.8M & 7.8 & 7.4 & 18.8 & 15.2 & 17.3  \\
Image-based $\cap$ CLIP score (50\%)  & 2.6M & \textbf{10.6} & \textbf{9.0} & 20.0 & \textbf{16.8} & 19.1  \\
$\mathbb{D}^2$ Pruning (image+text) & 2.8M & 7.2 & 7.2 & {19.3} & 13.8 & 17.3 \\
\midrule
\textbf{VAS} (ImageNet-1k) & 2.8M & {9.6} & {8.5} & {20.4} & \underline{16.6} & {18.6}  \\
\textbf{VAS} (ImageNet-1k) $\cap$ CLIP score (50\%) & 2.8M & {9.7} & {8.7} & \underline{22.1} & {16.0} & \textbf{19.5}\\

\textbf{VAS} (CC12M) $\cap$ CLIP score (50\%) & 2.8M & {9.1} & {8.5} & {21.3} & {15.9} & \underline{19.2}  \\
\textbf{VAS-D} (CC12M) $\cap$ CLIP score (50\%) & 2.8M & \underline{9.9} & \underline{8.9} & \textbf{22.5} & 16.2 & \textbf{19.5}  \\
\bottomrule
\end{tabular}
\vspace{-0.6em}
\end{table*}

\begin{table*}
\small
\centering
\caption{Ablation Study on the VAS and its variants on DataComp. All experiments commence by filtering out the top 45\% of data based on the CLIP score, followed by corresponding approaches to obtain 3.8M data.``image'' or ``text'' means using the variance of image or text embeddings to represent  $\bar\Sigma_{\text{test}}$, and ``image $\times$ text'' means representing $\bar\Sigma_{\text{test}}$ with the cross-covariance of image and text embeddings.}
\vspace{0.3em}
\label{tab:abla}
\begin{tabular}{@{}lccccc@{}}
\toprule
\textbf{Filtering Strategy} $\cap$ CLIP score (45\%)  &  \textbf{ImageNet} & \textbf{ImageNet Dist. Shift} & \textbf{VTAB} & \textbf{Retrieval} & \textbf{Average} \\ \midrule
Random Sampling & 4.2& 4.9& 17.2& 11.6& 15.6\\ \midrule
\textbf{VAS} (ImageNet-1k, image)  & \textbf{5.2} & \textbf{5.5} & \underline{19.0} & \textbf{12.2} & \textbf{17.4} \\
\textbf{VAS} (ImageNet-1k, text) & 3.9& 4.2& 16.3& 11.3& 14.9\\
\textbf{VAS} (ImageNet-1k, image $\times$ text) & 4.3& 4.9& 17.5& 11.8& 15.9\\
\midrule
\textbf{VAS} (DataComp, image) &  \underline{5.0} & \underline{5.4} & 17.9 & \underline{12.1} & 16.2 \\
\textbf{VAS} (DataComp, text) & 3.6& 4.1& 17.2& 10.9& 15.2\\
\textbf{VAS} (DataComp, image $\times$ text) & 4.0 & 4.6& 17.6& 11.7& 15.7\\
\midrule
\textbf{VAS-D} (DataComp, image) & 4.7 & \underline{5.4} & \textbf{19.7} & {11.7} & \underline{17.3} \\
\textbf{VAS-D} (DataComp, text) & 3.5 & 4.1& 16.7& 11.1& 15.4\\
\textbf{VAS-D} (DataComp, image $\times$ text) & 3.6& 4.2& 18.4& 11.1& 15.8\\
\bottomrule
\end{tabular}
\vspace{-1em}
\end{table*}
\vspace{-0.5em}
\section{Experiments}
\label{sec: experiment}
In this section, we evaluate the performance of the VAS-based strategy in Algo.~\ref{algo: main}. We first show the effectiveness of our ultimately proposed strategy and then present our ablation study result to further validate our three key designs. Please refer to Appendix~\ref{sub: algo details} for the implementation details and Appendix~\ref{sub: hyperp} for the discussion of why our algorithm can be used with little or no hyperparameter tuning.
\vspace{-0.8em}
\subsection{Setup}
\paragraph{Training Configurations.} 
We adhere to the standardized training and evaluation procedures as outlined in Datacomp~\cite{gadre2023datacomp}, considering two distinct training datasets. One is small-scale DataComp set, which provides a substantial dataset comprising 12.8 million low-quality, web-curated image-text pairs. Another is CC12M~\cite{changpinyo2021conceptual}, which offers a collection of 12 million high-quality image-text pairs.
Our comparative setup varies across these two datasets:
(1) Within the DataComp framework, our objective is to maximize model performance, without imposing a fixed dataset size constraint. Conversely, (2) for CC12M, where data quality is generally high and its removal could potentially degrade model performance, our focus shifts to the \textit{subset selection} challenge, which comparing the strategy under the same budget (approximately 2.8 million in our experiment)\footnote{Owing to the limitations of the heuristic method `Image-based filtering' cited in \cite{gadre2023datacomp} in controlling subset size, the size for the `Image-based filter $\cap$ CLIP score method (50\%)' is marginally below 2.8 million, as detailed in Tab.~\ref{tab:cc12m}}. Details in Appendix~\ref{sub: config}.
\vspace{-0.8em}
\paragraph{Evaluation Datasets.} 
We evaluate our model on 38 datasets related to image classification and retrieval tasks adopted in DataComp. The image classification tasks 
contain ImageNet-1k~\cite{deng2009imagenet}, ImageNet distribution shifts~\cite{wang2019learning,pmlr-v97-recht19a,hendrycks2021natural,hendrycks2021many}, 11 datasets from the Visual Task Adaptation Benchmark (VTAB)~\cite{zhai2019large} and 3 datasets from  WILDS~\cite{koh2021wilds, sagawa2021extending}, which tests the robustness facing subpopulation shift and spurious correlation. Retrieval datasets contain Flickr30k~\cite{young-etal-2014-image} and MSCOCO~\cite{chen2015microsoft} which focus on image-text retrieval and WinoGAViL~\cite{bitton2022winogavil} related to commonsense association. 
\subsection{Baselines}
We compare against several baselines from \citet{gadre2023datacomp} and \citet{maharana2023d2} in our paper. Here we give a summary for key baselines: (1) \textbf{CLIP score.} We simple select the data with top CLIP score generated by OpenAI's CLIP ViT-L/14 model, to reduce the size of training pool to 30\% and 45\% for DataComp and 30\% for CC12M. (2) \textbf{Image-based filtering.} we use the method proposed in \citet{gadre2023datacomp} which use external ImageNet-1K as data prior for selecting data. (3) \textbf{$\mathbb{D}^2$ Pruning.}  \citet{maharana2023d2} proposes a method that selects the data by combining the difficulty and diversity by representing the dataset as an undirected graph. They use the CLIP score to initialize their graph. Please refer to Appendix~\ref{sub: baselines} for details of baselines.

Here we reproduce CLIP score filtering and $\mathbb{D}^2$ pruning for DataComp in Tab.~\ref{tab:datacomp} because our reproducing results are different from the results of their original papers. The same reproducing issue has also been observed in \citet{maharana2023d2} and the reason may be that some URLs of data become invalid over time\footnote{
See https://github.com/mlfoundations/datacomp/issues/3
}, and the smaller dataset is more sensitive to this problem.

\subsection{Results and Discussions}

\paragraph{Results on DataComp and CC12M.}
We report our results for DataComp and CC12M in Tab.~\ref{tab:datacomp} and Tab.~\ref{tab:cc12m}, respectively. \textit{First}, compared with CLIP score or $\mathbb{D}^2$ pruning, directly applying VAS without using CLIP score performs significantly worse on DataComp but better on CC12M. This matches our expectations since VAS only measures the relevance of the visual information between training data and test prior, and doesn't consider the image-text quality. So using VAS alone should only perform well on high-quality datasets. 
\textit{Second}, combining VAS(-D) filter with the CLIP score filter will consistently outperform all the baselines among two datasets except the image-based $\cap$ CLIP score on CC12M. Compared with just utilizing the CLIP score filter, it improves 1.3\% and 2.2\% averagely on 38 evaluation sets for DataComp and CC12M, respectively, demonstrating its capability to select the most informative data and enhance model performance. 
\textit{Third}, note that VAS-D, the heuristic method we proposed without any requirements on the external dataset,
achieves similar good results to that with the variance of ImageNet-1k. Actually, in Appendix~\ref{sec: add_vis}, we further show that in many cases, VAS(Traindata) will have similar scores to VAS(ImageNet-1k), which suggests that the variance of the train dataset itself can indicate information about which features are more useful in downstream test distribution.
Besides, VAS-D(Traindata) is better than VAS(Traindata) slightly, which means that the dynamic update process of variance may have the effect of removing the unrelated or noisy features that may exist in the initial variance matrix and thus improve the quality of VAS. Finally, we also compare the time cost of different approaches in Appendix~\ref{sub: time cost}.


\vspace{-0.8em}
\paragraph{Ablation study on image and text embeddings.}
We have already show the effective of choosing image-only embeddings for VAS, here we test on two alternative approaches of choosing VAS, including: (1) text embedding only (denoted as `text'), which is equivalent to $\text{VAS}(l,l)$ as defined in Equation~(\ref{eq:def_VAS}) ; and (2) cross covariance on both image and test embeddings (denoted as `image $\times$ text'), which is equivalent to $\text{VAS}(v,l)$. The results shown in Table~\ref{tab:abla} substantiates our vision-only assertion in Section~\ref{subsec: main strategy} and theoretical analysis (Theorem~\ref{them: main}). Specifically, it shows that (1) using `text' VAS significantly underperforms compared to using `image'; (2) using `image $\times$ text' VAS shows a marginal performance improvement over the sole use of `text', but still falls short compared to using `image' alone.
\vspace{-0.8em}
\paragraph{Relation to optimal experimental design.}
Although our VAS filtering has achieved great results, it seems to have a conflict to the idea of optimal experimental design:
\vspace{-3px}
\begin{align*}
    \text{A-optimality} \quad S&= \arg \min_{|S|=N} \Tr(\Sigma_S^{-1})\\
    \text{V-optimality} \quad S&= \arg\min_{|S|=N} \Tr(\Sigma_{\text{prior}}\Sigma_S^{-1})
\end{align*}
\begin{table}[t]
\small
\centering
\vspace{-0.5em}
\caption{Relation to Optimal Experimental Design on CC12M. All experiments firstly filter out 50\% data with the top CLIP score and then use corresponding methods to obtain 2.8M data.}
\vspace{0.5em}
\begin{tabular}{@{}lcccc@{}}
\toprule
\textbf{Filtering Strategy}  & \textbf{ImageNet} & \textbf{VTAB} & \textbf{Retrieval} & \textbf{Avg.}\\ \midrule
CLIP score & 7.8& 18.8& 15.2& 17.3\\
Random Sampling & 8.5& 19.9& \textbf{16.3}& 18.1\\
\textbf{VAS-D} (image) & \textbf{9.9}& \textbf{22.5}& \underline{16.2} & \textbf{19.5}\\
\midrule
A-Optimal (image) & 8.2 & 19.2 & 15.6 &17.5\\ 
V-Optimal (image) & 5.1& 16.8& 14.8&15.7  \\
A-Optimal (text) & \underline{8.5} & \underline{20.7} & 15.8 & \underline{18.4}\\
V-Optimal (text) & 7.8 & 19.9 & 15.7 & 18.0\\
\bottomrule
\end{tabular}
\label{tab:oed}
\vspace{-1em}
\end{table}
Here $\Sigma$ is the variance matrix. We choose embeddings from ImageNet-1k as the test prior and compare the experimental design algorithms with our VAS-D in Tab.~\ref{tab:oed}. We can find that these methods perform badly, and V-optimality (image) even performs much worse than vanilla CLIP score filtering. Interestingly, optimal experiment design with text embeddings performs better than that with image embeddings, although they are still significantly inferior to our methods.

\section{Conclusion}
In the paper, we propose a Variance Alignment Score (VAS) to measure the informativeness of each sample and show that VAS is a parallel metric to image-text alignment focused score like CLIP score. Based on VAS, we design a informative + quality focused strategy that achieves the best average performance over 38 downstream test tasks compared to various existing baseline. Moreover, we also provide a theoretical guarantees to explain the intuition behind our strategy.

\newpage

\bibliography{example_paper}
\bibliographystyle{icml2024}

\newpage
\appendix
\onecolumn

\section{Detailed proofs}
\label{app: proof}

\begin{lemma}
\label{lemma:EYM_formal}
    Let
    \begin{equation}
        \hat G_v, \hat G_l = \arg \min_{G_v, G_l \in \R^{d\times r}}\mathcal{L}(G_v,G_l)
    \end{equation}
    Then we have
    \begin{equation}
        \hat G_v \hat G_l^\top = \frac{1}{\rho}G_v^* \Sigma_S (G_l^*)^\top + P_1 + P_2 + P_3 + P_4
    \end{equation}
    where noise terms $P_i$ are defined in (\ref{eq:P1}) , (\ref{eq:P2}), (\ref{eq:P3}) and (\ref{eq:P4}).
\end{lemma}
\begin{proof}
    Note that $s_{ij} = (\vx_j^{l})^\top G_l G_v^\top \vx_i^{v}= \Tr(G_v^\top \vx_i^{v}(\vx_j^{l})^\top G_l)$, like the proof of Corollary B.1. in \citet{nakada2023understanding}, we have
    \begin{eqnarray}
        \mathcal{L}(G_v, G_l) &=& \frac{\sum_{i\in S}\sum_{j\in S}(s_{ij} - s_{ii})}{|S|(|S|-1)} 
    + \frac{\rho}{2}\frac{|S|}{|S|-1}\|G_vG_l^\top\|_F^2\\
    &=& \frac{\sum_{i\in S}\sum_{j\in S}s_{ij} - |S|\sum_{i\in S} s_{ii}}{|S|(|S|-1)} 
    + \frac{\rho}{2}\frac{|S|}{|S|-1}\|G_vG_l^\top\|_F^2\\
    &=& -\Tr\left(G_v^\top\left[
    \frac{1}{|S|-1}\sum_{i\in S} \vx_i^{v}(\vx_i^{l})^\top - \frac{|S|}{|S|-1}\bar\vx^{v}(\bar\vx^{l})^\top
    \right]G_l\right)+ \frac{\rho}{2}\frac{|S|}{|S|-1}\|G_vG_l^\top\|_F^2\\
    &=:& -\Tr(G_v^\top \Gamma G_l) + \frac{\rho}{2}\frac{|S|}{|S|-1}\|G_vG_l^\top\|_F^2
    \end{eqnarray}
    where $\bar\vx^{vl}:= (\sum_{i\in S}\vx_i^{vl})/|S|$. Then by the Eckart-Young-Mirsky Theorem (For example, Theorem 2.4.8 in \citet{golub2013matrix}), we know that
    \begin{eqnarray}
        &&\arg\min_{G_v \in \R^{d\times r}, G_l \in \R^{d\times r}} \mathcal{L}(G_v, G_l)\\
        &=& \arg\max_{G_v \in \R^{d\times r}, G_l \in \R^{d\times r}} \Tr(G_v^\top \Gamma G_l) - \frac{\rho}{2}\frac{|S|}{|S|-1}\|G_vG_l^\top\|_F^2\\
        &=& \{(G_v, G_l) \in \R^{d\times r}\times \R^{d\times r}: G_vG_l^\top = \frac{1}{\rho}\frac{|S|-1}{|S|}\mathrm{SVD}_r(\Gamma)\} \qquad (\text{Eckart-Young-Mirsky Theorem})
    \end{eqnarray}
    where the notation $\mathrm{SVD}_r(\Gamma)$ means choosing the first $r$ components of the matrix $\Gamma$. Further note that
    \begin{eqnarray}
        \Gamma &=& \frac{1}{|S|-1}\sum_{i\in S} \vx_i^{v}(\vx_i^{l})^\top - \frac{|S|}{|S|-1}\bar\vx^{v}(\bar\vx^{l})^\top\\
        &=:& P_0 + P_1 + P_2 + P_3 + P_4
    \end{eqnarray}
    Here note that $\Sigma_S =\frac{1}{|S|}\sum_{i \in S}\vz_i^{v}(\vz_{i}^{l})^\top$, we have $P_i$ as follows: 
    \begin{eqnarray}
        P_0 &:=& \frac{|S|}{|S|-1}G_v^* \cdot \Sigma_S \cdot (G_l^*)^\top\\
        P_1 &:=& \frac{1}{|S|-1} G_v^* \sum_{i\in S}\vz_i^{v}(\vxi_i^{l})^\top\label{eq:P1}\\
        P_2 &:=& \frac{1}{|S|-1}\sum_{i\in S}\vxi_{i}^{v}(\vz_{i}^{l})^\top  (G_l^*)^\top\label{eq:P2}\\
        P_3 &:=& \frac{1}{|S|-1}\sum_{i\in S} \vxi_i^{(1)}(\vxi_i^{(2)})^\top\label{eq:P3}\\
        P_4 &:=& -\frac{|S|}{|S|-1}\bar\vx^{v}(\bar\vx^{l})^\top\label{eq:P4}
    \end{eqnarray}
    It's clear that the rank of the matrix $P_0$ is no more than $r$, so $\mathrm{SVD}_r(P_0) = P_0$. And for $i \in \{1,2,3,4\}$, $P_i$ are noise terms with $\E[P_i] = O$.
\end{proof}

\begin{lemma}
    For any fixed $S$, w.h.p $1-\delta$ the noise term can be upper bounded by 
    $\sqrt{\frac{d\log(1/\delta)}{|S|}}$
\end{lemma}
\begin{proof}
    To upper bound the P1 and P2, we have
   \begin{align*}
    & \| \sum_{i} \vz_i^{vl}(\xi_i^{vl})^\top \|_*^2
    = \Tr\left( \sum_{i,j} \xi_i^{vl} (\vz_i^{vl})^\top \vz_j^{vl} \xi_j^{vl}\right) 
    = \sum_{i,j} (\vz_i^{vl})^\top \vz_j^{vl} (\xi_j^{vl})^\top \xi_i^{vl} \\
    & \E \| \sum_{i} \vz_i^{vl}(\xi_i^{vl})^\top \|_*^2 
    = \E \left[\sum_{i} (\vz_i^{vl})^\top \vz_i^{vl} (\xi_i^{vl})^\top \xi_i^{vl} \right]
    = |S| d
    \end{align*}
    Regarding each $(\vz_i^{vl})^\top \vz_j^{vl} (\xi_j^{vl})^\top \xi_i^{vl}$ as weakly dependent variable, then by using Bernstein inequality, we have, with high probability $1-\delta$,
    \begin{align*}
         \| \sum_{i} \vz_i^{vl}(\xi_i^{vl})^\top \|_*^2
         \leq |S|d + \sqrt{d |S|^2 \sigma_\xi^2\log(1/\delta)} 
         \leq |S|d\sqrt{\log(1/\delta)}
    \end{align*}
    So 
    $\frac{1}{|S|}\| \sum_{i} \vz_i^{vl}(\xi_i^{vl})^\top \|_* \leq \sqrt{\frac{d\log(1/\delta)}{|S|}}$. Note that $\|\bar\vx^{vl}\| \lesssim \sqrt{\frac{\log(|S|d)}{|S|}}$ (like Proposition 2.5 in \citet{wainwright2019high}), it is easy to see that P3 ad P4 are the low order terms if $\delta \lesssim \frac{1}{|S|d}$.
\end{proof}

\begin{lemma}[Intuition behind VAS]
    With high probability $1-\delta$, suppose the  hind-side best subset has at least $\underline{n}$ number of samples, then we have 
    \begin{align*}
        \Delta(S) 
        = \frac{1}{\rho}\max_{S' \in D_\text{train}} \left( \Tr\left( \Sigma_{\text{test}} (\Sigma_{S'} - \Sigma_{S}) \right) \right) + \sqrt{\frac{d\log(1/\delta)}{\underline{n}}} + \sqrt{\frac{d\log(1/\delta)}{|S|}}
    \end{align*}
\end{lemma}
\begin{proof}
    For any learnt $G_v, G_l$ based on dataset $S$, we have 
    \begin{align*}
        \gL_\text{test}(G_v, G_l) 
        & = \Tr( G_v^\top\E_{\vx_{vl} \sim \gD_{\text{test}}}[\vx^v (\vx^l)^\top] G_l  ) \\
        & = \Tr( \E_{\vx_{vl} \sim \gD_{\text{test}}}[\vx^v (\vx^l)^\top] G_l G_v^\top ) \\
        & = \frac{1}{\rho}\Tr\left(\E_{\vx_{vl} \sim \gD_{\text{test}}}[\vx^v (\vx^l)^\top] G_l^*\Sigma_S (G_v^*)^\top\right)
            - \Tr\left(\E_{\vx_{vl} \sim \gD_{\text{test}}}[\vx^v (\vx^l)^\top] \text{noise}_S\right) \\
        &= \frac{1}{\rho} \Tr\left( (G_v^*)^\top \E_{\vx_{vl} \sim \gD_{\text{test}}}[\vx^v (\vx^l)^\top] G_l^*\Sigma_S \right)
            - \Tr\left( \E_{\vx_{vl} \sim \gD_{\text{test}}}[\vx^v (\vx^l)^\top] \text{noise}_S\right) \\
        & = - \frac{1}{\rho} \Tr\left( \Sigma_{\text{test}}\Sigma_S \right) 
            - \Tr\left( \E_{\vx_{vl} \sim \gD_{\text{test}}}[\vx^v (\vx^l)^\top] \text{noise}_S\right)
    \end{align*} 
    Here the first equation comes from Theorem~\ref{them: simplified test loss} and the third equation comes from Lemma~\ref{lemma:EYM_formal}. \\
    Consequently, we have 
    \begin{align*}
        - \min_{S' \in D_\text{train}} \gL_\text{test}(\gA(S'))
        & = \max_{S' \in D_\text{train}} \left( \frac{1}{\rho}\Tr\left( \Sigma_{\text{test}}\Sigma_{S'} \right)  + \Tr\left( \E_{\vx_{vl} \sim \gD_{\text{test}}}[\vx^v (\vx^l)^\top] \text{noise}_{S'}\right) \right) \\
        & \leq \frac{1}{\rho}\max_{S' \in D_\text{train}} \left( \Tr\left( \Sigma_{\text{test}}\Sigma_{S'} \right) \right) 
            + \|\E_{\vx_{vl} \sim \gD_{\text{test}}}[\vx^v (\vx^l)^\top]\| \|\text{noise}_{S'}\|_* \\
        & \leq \frac{1}{\rho}\max_{S' \in D_\text{train}} \left( \Tr\left( \Sigma_{\text{test}}\Sigma_{S'} \right) \right) 
                    + \gO\left(\sqrt{\frac{d\log(1/\delta)}{\underline{n}}} \right)
    \end{align*}
    Therefore, we have the final result as 
    \begin{align*}
        \Delta(S) 
        & = \gL_\text{test}(\hat{G}_{vl}) - \min_{S' \in D_\text{train}} \gL_\text{test}(\gA(S'))\\
        & = \frac{1}{\rho}\max_{S' \in D_\text{train}} \left( \Tr\left( \Sigma_{\text{test}} (\Sigma_{S'} - \Sigma_{S}) \right) \right) + \gO\left(\sqrt{\frac{d\log(1/\delta)}{\underline{n}}} + \sqrt{\frac{d\log(1/\delta)}{|S|}} \right)
    \end{align*}
\end{proof}

\begin{theorem}[Main]
\label{them: main (app)}
    Under the assumption of Lemma~\ref{lem: main}, we have
    \begin{align*}
       \Delta(S)
        & \leq \text{noise} +  \|\bar{\Sigma}_{\text{test}} - \Sigma_{\text{test}} \| \|\Sigma_S - \Sigma_{\text{best}}\|_* \\
        & + 
        \begin{cases}
            \epsilon_{v * l}^S  \quad \text{(vision+language)} \\
            \left( \epsilon_v^S +  \sqrt{1 - \frac{1}{|S|}\sum_{i \in [S]} \langle \vz^v,\vz^l \rangle)}\right) \quad \text{(vision only)}
        \end{cases}
    \end{align*}
\end{theorem}
\begin{proof}
    Based on Lemma~\ref{lem: main}, we will focus on the error cause from selecting subset $S$, that is, $\Tr \Sigma_{\text{test}}\Sigma_S$. Since the exact $\Sigma_{\text{test}}$ is unknown, we assume the access to some proxy $\bar{\Sigma}_{\text{test}}$ instead.

    Recall that, for any $S$, we have ground-truth $\Sigma_S = \E_{\vz_{vl} \in S} \vz^v (\vz^l)^\top$. Unfortunately, this is not directly observable by the learner. Instead, the learner is able to observe some proxy  $\bar{\Sigma}_S$ based on the teacher model $\bar{G}_{vl}$ and therefore solving
    \begin{align*}
        \argmax_{S} \Tr\left(\bar{\Sigma}_{\text{test}}\bar{\Sigma}_S\right)
    \end{align*}
    and therefore, denote $\Sigma_\text{best} = \argmax_{S' \in D_\text{train}} \Tr\left( \Sigma_{\text{test}} \Sigma_{S'} \right)$
    \begin{align*}
         \Tr \left( \Sigma_{\text{test}}(\Sigma_\text{best} - \Sigma_S) \right)
         & = \Tr \left( \bar{\Sigma}_{\text{test}}(\Sigma_\text{best} - \bar{\Sigma}_S) \right)
             + \Tr \left( \bar{\Sigma}_{\text{test}}(\bar{\Sigma}_S - \Sigma_S) \right)
             + \Tr \left( (\Sigma_{\text{test}} - \bar{\Sigma}_{\text{test}})(\Sigma_\text{best} - \Sigma_S) \right) \\
        & \leq \Tr \left( \bar{\Sigma}_{\text{test}}(\bar{\Sigma}_S - \Sigma_S) \right)
             + \Tr \left( (\Sigma_{\text{test}} - \bar{\Sigma}_{\text{test}})(\Sigma_\text{best} - \Sigma_S) \right) \\
        & \leq \| \Sigma_{\text{test}}\| \|\bar{\Sigma}_S - \Sigma_S\|_* 
            +  \|\bar{\Sigma}_{\text{test}} - \Sigma_{\text{test}} \| \|\Sigma_S - \Sigma_{\text{best}}\|_*
    \end{align*}
    where the first inequality is by the definition of $\bar{\Sigma}_S$ and the second inequality comes from holder's inequality.
    Now the key is to upper bound $\|\bar{\Sigma}_S - \Sigma_S\|_*$ based on our chosen strategy.

    In option 1, we use the clip embedding from both visual and language modal. That is, choose 
    $\bar{\Sigma}_S = \sum_{\vx_{vl} \in S} (\bar{G}_v)^\top \vx^v (\vx^l)^\top \Bar{G}_l$. Then we have 
    \begin{align*}
        \|\bar{\Sigma}_S - \Sigma_S\|_*
        \leq  \frac{1}{|S|}\| \sum_{\vx_{vl} \in S} (\bar{G}_v)^\top \vx^v (\vx^l)^\top \Bar{G}_l - \sum_{\vx_{vl} \in S} \vz^v (\vz^l)^\top \|_* 
        \leq \epsilon_{v * l}^S
    \end{align*}

    In option 2, we use the clip embedding from language model only. That is choose 
    $\bar{\Sigma}_S = \sum_{\vx_{vl} \in S} \bar{G}_v^\top \vx^v (\vx^v)^\top \Bar{G}_v$. Then, by definition of $\epsilon_S$, we have 
    \begin{align*}
        \|\bar{\Sigma}_S - \Sigma_S\|_*
        & \leq  \frac{1}{|S|}\| \sum_{\vx_{vl} \in S} \bar{G}_v^\top \vx^v (\vx^v)^\top \Bar{G}_v - \sum_{\vx_{vl} \in S} \vz^v (\vz^v)^\top \|_* + \frac{1}{|S|}\| \sum_{\vx_{vl} \in S} \vz^v (\vz^v)^\top -  \Sigma_S\|_* \\
        & \leq \epsilon_v^S +  \frac{1}{|S|}\| \sum_{\vx_{vl} \in S} \vz^v (\vz^v)^\top -  \Sigma_S\|_*
    \end{align*}
    Now to further bound the second term, we have
    \begin{align*}
        \frac{1}{|S|} \| \sum_{\vx_{vl} \in S} \vz^v (\vz^v)^\top -  \Sigma_S\|_* 
        & \leq \frac{1}{|S|} \|Z_v^\top \|_* \|Z_v - Z_l\|_*\\
        & = \frac{1}{|S|} \sqrt{\Tr Z_v Z_v^\top} \sqrt{ \Tr (Z_v - Z_l)^\top (Z_v - Z_l)} \\
        & = \frac{1}{|S|} \sqrt{\Tr(I_{n\times n})} \sqrt{2\Tr\left( I_{n\times n} - Z_v Z_l^\top\right)}\\
        & = \frac{1}{|S|} \sqrt{2|S|(|S| - \sum_{i \in [S]} \langle \vz^v,\vz^l \rangle)} \\
        & = \sqrt{1 - \frac{1}{|S|}\sum_{i \in [S]} \langle \vz^v,\vz^l \rangle)}
    \end{align*}
    Therefore, we finish the proof.
\end{proof}

\begin{theorem}[A simplified version of test loss]
\label{them: simplified test loss}
    Under the assumption that both $\vz_{vl}, \xi_{vl}$ is zero-mean, maximizing the clip score gap is equivalent to maximize the clip score of the same sample. 
     \begin{align*}
         \gL_{\text{test}} (G_v, G_l):= - \E_{\vx_{vl} \sim \gD_\text{test}}\langle G_v^\top\vx_v, G_l^\top \vx_l \rangle
     \end{align*}
\end{theorem}
\begin{proof}
    For any $\vx_{vl}$, we have
    \begin{align*}
        & \E_{\vx_{vl}' \sim \gD_\text{test}} (\langle G_v^\top\vx_v, G_l^\top\vx_l'\rangle- \langle G_v^\top\vx_v, G_l^\top\vx_l\rangle) \\
        & = \langle G_v^\top\vx_v, G_l^\top \E_{\vx_{vl}' \sim \gD_\text{test}}(\vx_l'-\vx_l) \rangle \\
        & = - \langle G_v^\top\vx_v, G_l^\top \vx_l \rangle \\
    \end{align*}
\end{proof}

\section{Details of Experiments}

\subsection{Training Configuration} \label{sub: config}
For the training sets, we choose ViT-B/32 model as the baseline, with training computations fixed at $9.5\times 10^{16}$ total multiply-accumulate operations (MACs), in line with the Datacomp~\cite{gadre2023datacomp} standards. All experiments were conducted on NVIDIA A40 GPUs.


\subsection{Baselines}\label{sub: baselines}
We add some details about the baselines used in our paper.
\begin{itemize}[leftmargin=*]
    \item \textbf{Text-based filtering.}  \citet{gadre2023datacomp} proposes a text-based filtering that tries to select the data that contains caption overlapping with the class name from ImageNet-21K or ImageNet-1K. We just apply this filter on DataComp because this heuristic method cannot determine the size of the selected pool of data, and thus not suitable for the coreset selection scenario on CC12M.
    \item \textbf{Image-based filtering.}  \citet{gadre2023datacomp} also proposes a heuristic way to sample the visual content overlaps with ImageNet-1K classes. They first apply filtering by language (only choose English caption by fasttext~\cite{joulin2016bag}) and caption length (over two words and 5 characters). Then they cluster the image embeddings from training data to 100K groups using Faiss~\cite{johnson2019billion}, and keep the groups whose cluster center is the nearest neighbor to at least one image embedding of ImageNet-1K image. 
    \item \textbf{$\mathbb{D}^2$ Pruning.}  \citet{maharana2023d2} tries to represent the dataset as an undirected graph for coreset selection. They assign the difficulty for each example and use message passing to update the difficulty score incorporating the difficulty of its neighboring examples, and finally try to keep both diverse and difficult subsets. For our experiments, we adhere to the default hyperparameters of $\mathbb{D}^2$ on DataComp as specified in their official codebase. Additionally, we tune the number of nearest neighbors over $k = \{1,5,10,15\}$ and report their best results on CC12M, following their recommendations in the paper.
\end{itemize}
\subsection{Algorithm Details}\label{sub: algo details}
In this section, we illustrate the details of our VAS-D algorithm, the generation of text embeddings of ImageNet-1k, and the choice of hyperparameters.

\paragraph{VAS-D.} We use greedy remove methods to achieve the object (\ref{eq:tr_S2}). Note that if the number of greedy steps is $\tau$, and let $\bar\Sigma_{\text{test},i} = \frac{1}{|S_i|}\sum_{j \in S_i}\bar f_v(\vx_j^v)\bar f_v(\vx_j^v)^\top$ where $S_i$ is the selected subset at step $i$, then obtaining (\ref{eq:tr_S2}) with step-by-step greedy remove is equivalent to removing the data satisfies (\ref{eq:greedy}):
\begin{equation}\label{eq:greedy}
    S_{i} \setminus S_{i+1} = \arg \max_{p \in S_{i}} \Tr\left((\bar \Sigma_{\text{test},i} - f_v(\vx_p^v)\bar f_v(\vx_p^v)^\top)^2\right) = \arg \min_{p \in S_{i}} \left[\bar f_v(\vx_p^v)^T \cdot\bar \Sigma_{\text{test},i}\cdot\bar f_v(\vx_p^v)\right], \quad i \in \{0, \ldots, \tau-1\}
\end{equation}
Then we can detail the algorithm process of VAS-D in Algorithm~\ref{algo: VAS-D}. In experiments, we all fix the number of greedy step $\tau$ to $168$ on both DataComp and CC12M. But this is not a parameter that needs to be finely tuned. Generally, the selected subset is well when $\tau > 100$ for DataComp and CC12M. 

\begin{algorithm}
\caption{VAS-D with greedy remove strategy}
\label{algo: VAS-D}
\begin{algorithmic}
    \STATE {\bfseries Inputs:} image embeddings of the data after CLIP score filtering $\{\bar f_v(\vx_i^v)\}_{i \in S}$, target size $N$, number of greedy steps $\tau$
    \STATE Initialize $S_0 = S, N_0 = |S|$
    \FOR{$t=1$ {\bfseries to} $\tau$}
    \STATE Size at step $t$ : $N_t = N_0 - \frac{t}{\tau}(N_0 - N)$.
    \STATE Prior matrix: $\bar\Sigma_{\text{test}, t-1} = \sum_{j \in S_{t-1}}\bar f_v(\vx_j^v)\bar f_v(\vx_j^v)^\top$
    \STATE Updated VAS for each sample $p$ in $S_{t-1}$: $\text{VAS}_p = \bar f_v(\vx_p^v)^\top\cdot \bar\Sigma_{\text{test}, t-1}\cdot \bar f_v(\vx_p^v)$
    \STATE Construct $S_t$ such that it contains the data with highest VAS in $S_{t-1}$ and satisfies $|S_t| = N_t$. 
    \ENDFOR
\end{algorithmic}
\end{algorithm}


\paragraph{Text embeddings of ImageNet-1k.}
In the ablation study (Tab.~\ref{tab:abla}), we use text embedding from ImageNet-1k to calculate VAS. There we select 80 templates as \cite{radford2021learning} for prompt generation for each class and take the mean of their embeddings as the representative text embedding for images within that class. 

\subsection{Hyperparameters}\label{sub: hyperp}


The main hyper-parameters of our algorithm (Alg.~\ref{algo: main} and Alg.~\ref{algo: VAS-D}) are the target numbers after CLIP score filtering and VAS filtering. In this paper, we choose CLIP score filtering to get around 50\% of data from the original pool and the VAS filtering to further get 30\% of data from the original pool . This is approximately equivalent to setting CLIP score threshold as \textbf{0.214} and VAS threshold as \textbf{0.153}. 
To give readers a more intuitive understanding of how these two thresholds select most of the high-quality and informative data, we provide a visualization with different CLIP scores and VAS scores in Fig.~\ref{fig:vis_many}.

Notably, the choice of both thresholds is \textbf{train-dataset-independent}. That is, the CLIP score only evaluates individual data quality without considering the whole train data distribution;  Similarly, VAS, according to our definition, is only dependent on individual data embedding and the test prior (ImageNet-1k in our case). Such independence suggests that, when the teacher embedding model and the test prior is fixed, our threshold choice can be potentially transferred to other train datasets. So we recommend other researchers to try the initial threshold at around these values when executing our algorithm.

In addition, for VAS-D where the test-prior is highly related to the training dataset, there also exists an easy way to decouple. Specifically, we can use VAS(ImageNet-1k) as data prior to estimate the target number after VAS-D filtering and then apply VAS-D(Traindata) as Alg.~\ref{algo: VAS-D}. In this way, our methods can also be used with little or no parameter tuning.

\subsection{Time Cost}\label{sub: time cost}
We compare the time cost for different methods on 1 NVIDIA A40 GPU in Tab.~\ref{tab:time}. Here since all the key approaches in this paper require obtaining embeddings, so we don't take the time of calculating image/text embeddings with the pretrained CLIP model into consideration. We can observe that our VAS(ImageNet-1k/Dataset) + CLIP score filtering is extremely time efficient. Although VAS-D takes longer time due to the multi-step greedy removal process, it's acceptable compared with total training time and still faster than image-based filtering which requires Faiss~\cite{johnson2019billion} for clustering. 
\begin{table}
\small
\centering
\caption{Approximate time cost for different methods on DataComp with 1 NVIDIA A40.}
\label{tab:time}
\begin{tabular}{@{}lc@{}}
\toprule
  & Time Cost\\ \midrule
CLIP Score  & $<$\textbf{1m}\\
Image-based & 1h\\
$\mathbb{D}^2$ Pruning & 25m\\
VAS(ImageNet-1k/DataComp) & $<$\textbf{1m}\\
VAS-D(DataComp) &  36m\\
\midrule
Training Time & 16h\\
\bottomrule
\end{tabular}
\end{table}

\section{Additional Visualization}\label{sec: add_vis}
We further visualize more data with different CLIP scores and VAS in Fig.~\ref{fig:visual} and Fig.~\ref{fig:vis_many}. In Fig.~\ref{fig:vis_many} we can see that VAS(ImageNet-1k) always have similar score as VAS(DataComp).

Furthermore, we randomly sample 5000 data points from DataComp and calculate their CLIP score and VAS and scatter them in Fig.~\ref{fig:scatter}, we can see that after filtering out the data with top 50\% CLIP score, data with lower VAS may have higher probability for achieving high CLIP score, and data with higher VAS generally have a more centralized CLIP score range.

\begin{figure*}[h]
\centering
    \caption{Data distribution on VAS and CLIP score. We randomly sample 5000 points in DataComp and show its corresponding VAS and CLIP score, here VAS is calculated by the image embeddings from ImageNet-1k.}
    \includegraphics[width=0.5 \textwidth]{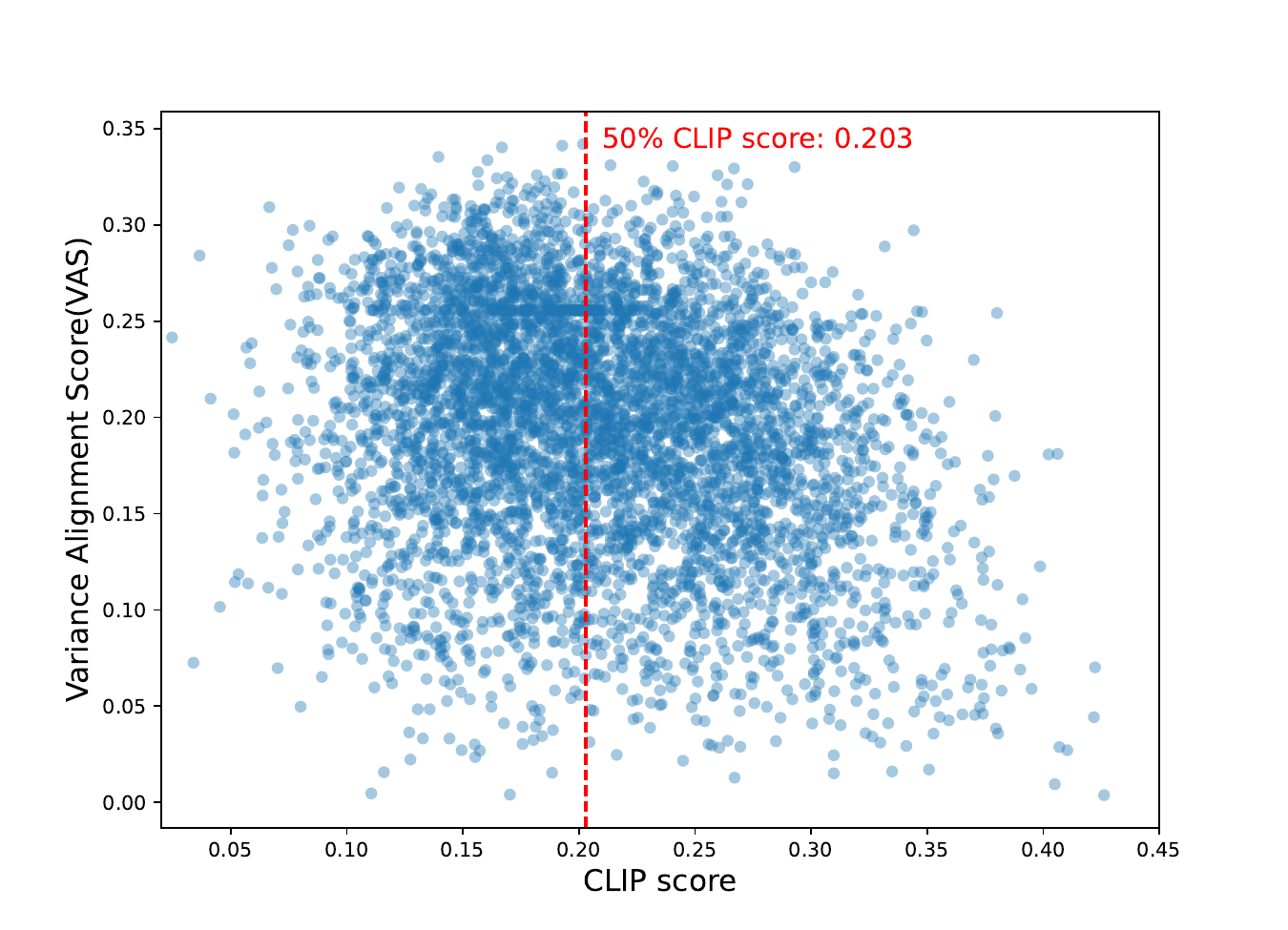}
    \label{fig:scatter}
    \vspace{-1em}
\end{figure*}

\begin{figure*}[h]
\centering
    \caption{Visualization of image data with different CLIP scores and VAS in DataComp.}
    \includegraphics[width=0.80 \textwidth]{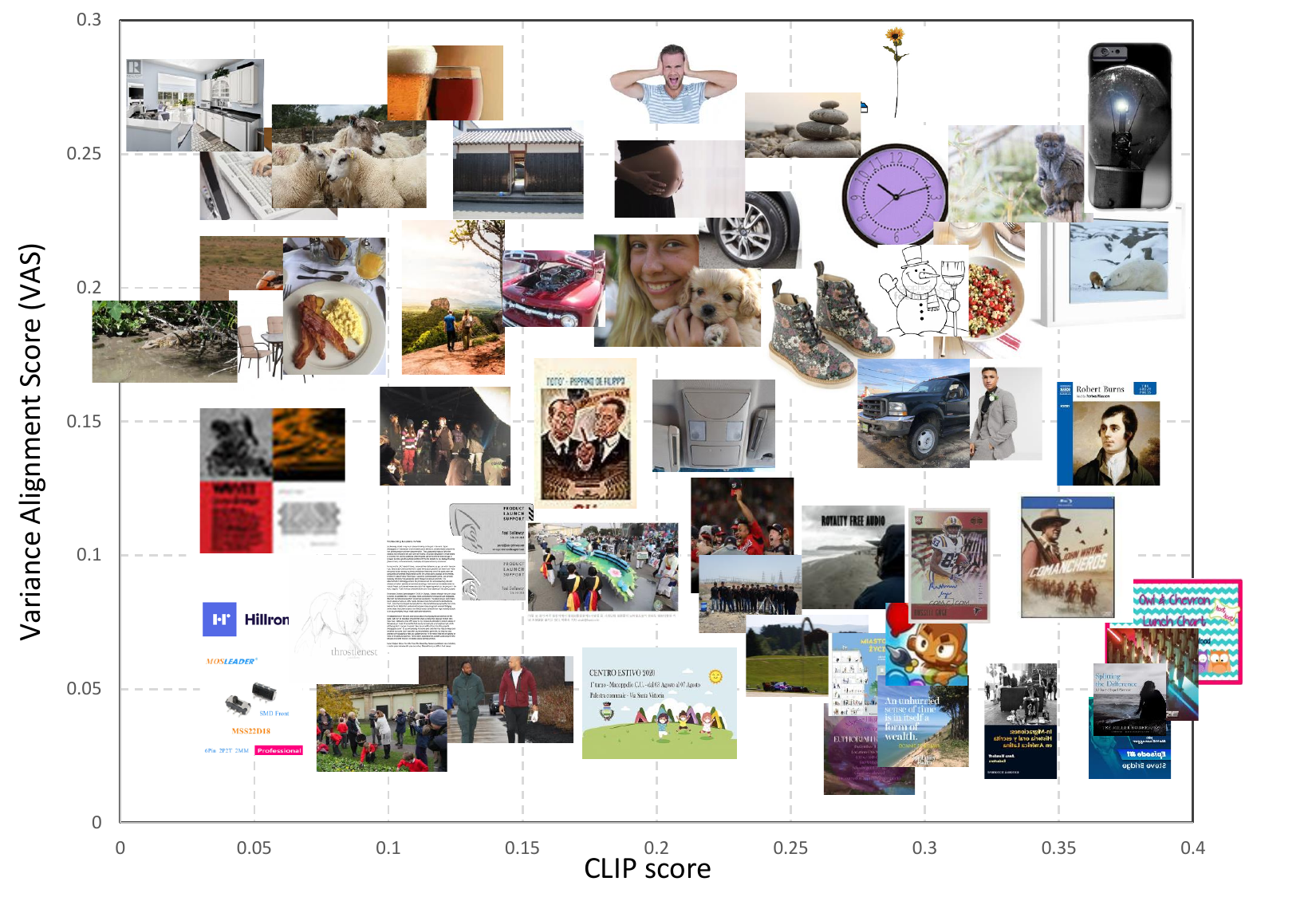}
    \label{fig:visual}
\end{figure*}


\newpage

\begin{figure*}[t]
    \centering
    \caption{Visualization of data pairs with different CLIP scores and VAS in DataComp. Here `img1k\_vas' means VAS(ImageNet-1k) and `self\_vas' denotes VAS(DataComp). We can see that for most of the data, VAS(DataComp) is always similar to VAS(ImageNet-1k).}
    \rotatebox[]{270}{\includegraphics[width=1.2 \textwidth]{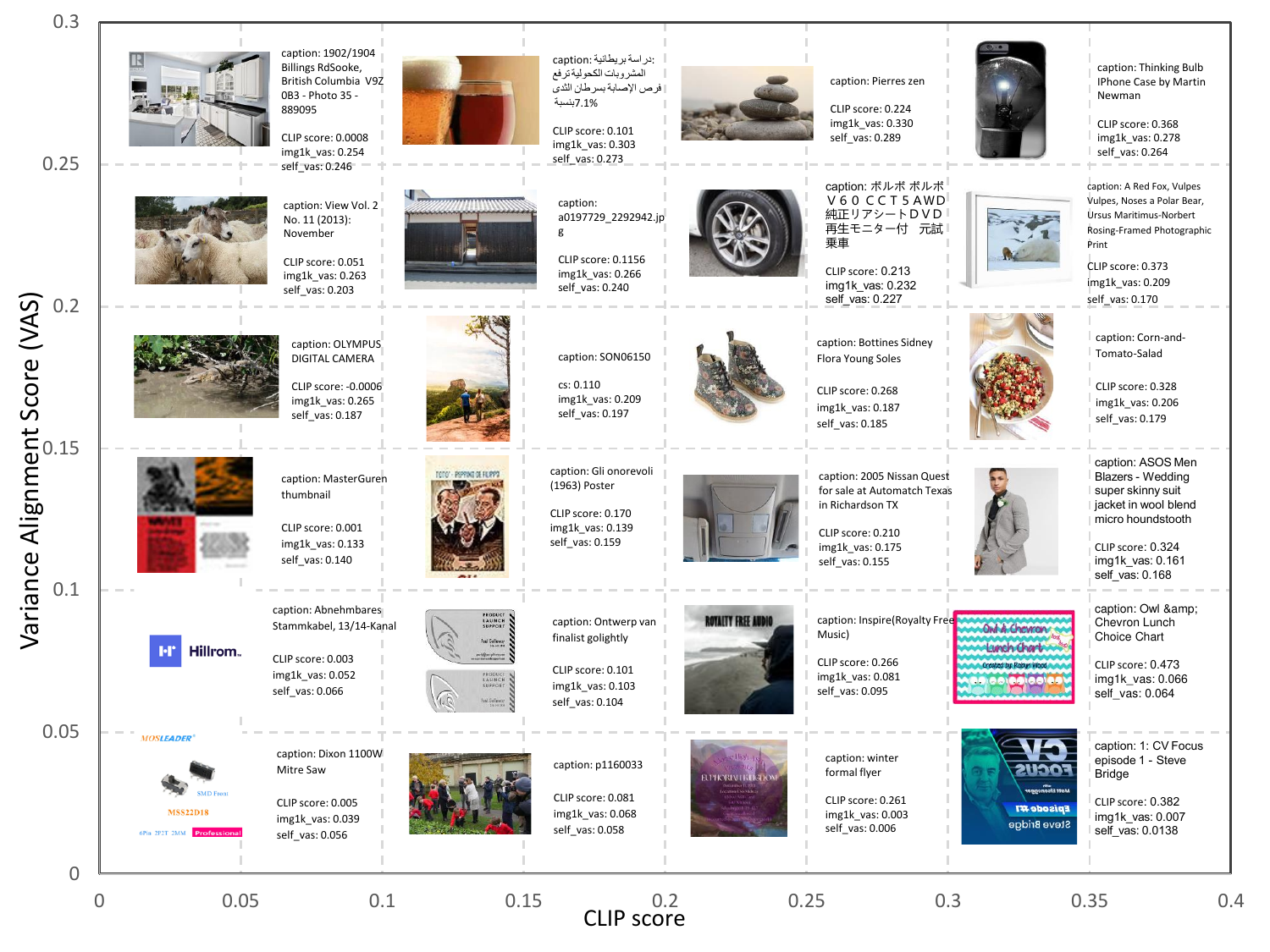}}
    \vspace{-1.0em}
    \label{fig:vis_many}
    \vspace{-0.7em}
\end{figure*}

\end{document}